\documentclass[conference]{IEEEtran}
\usepackage[utf8]{inputenc}
\IEEEoverridecommandlockouts
\usepackage{cite}
\usepackage{amsmath,amssymb,amsfonts}
\usepackage{algpseudocode}
\usepackage{graphicx}
\usepackage{textcomp}
\usepackage{xcolor}
\def\BibTeX{{\rm B\kern-.05em{\sc i\kern-.025em b}\kern-.08em
    T\kern-.1667em\lower.7ex\hbox{E}\kern-.125emX}}
    
\usepackage{tikz}  
\usetikzlibrary{decorations.markings}
\tikzstyle{vertex}=[circle, draw, inner sep=0pt, minimum size=6pt]
\ifCLASSOPTIONcompsoc
\usepackage[caption=false,font=normalsize,labelfont=sf,textfont=sf]{subfig}
\else
\usepackage[caption=false,font=footnotesize]{subfig}
\fi

\usepackage{amsthm}
\usepackage[utf8]{inputenc}
\usepackage[english]{babel}
\addto\captionsenglish{\renewcommand*{\tablename}{TABLE}}

\newtheorem{theorem}{Theorem}[section]
\theoremstyle{definition}
\newtheorem{definition}{Definition}[section]
\newtheorem{lemma}[theorem]{Lemma}

\usepackage{enumitem}

\algnewcommand{\algorithmicgoto}{\textbf{go to}}%
\algnewcommand{\Goto}[1]{\algorithmicgoto~\ref{#1}}%

\usepackage{pgfplots}
\pgfplotsset{compat=1.17}
\usepackage{array}

\begin{document}

\title{cgSpan: Closed Graph-Based Substructure Pattern Mining}

\author{\IEEEauthorblockN{Zevin Shaul}
\IEEEauthorblockA{
\textit{Informatica}\\
szevin@informatica.com}
\and
\IEEEauthorblockN{Sheikh Naaz}
\IEEEauthorblockA{
\textit{University of Wisconsin–Madison}\\
naazish.sheikh@gmail.com}
}
\IEEEpubid{\makebox[\columnwidth]{978-1-6654-3902-2/21/\$31.00~\copyright~2021 IEEE \hfill} \hspace{\columnsep}\makebox[\columnwidth]{ }}
\maketitle
\begin{abstract}
gSpan is a popular algorithm for mining frequent subgraphs. cgSpan (\underline{c}losed \underline{g}raph-based \underline{s}ubstructure \underline{pa}tter\underline{n} mining) is a gSpan extension that only mines closed subgraphs. A subgraph g is closed in the graphs database if there is no proper frequent supergraph of g that has equivalent occurrence with g.
cgSpan adds the Early Termination pruning method to the gSpan pruning methods, while leaving the original gSpan steps unchanged.

cgSpan also detects and handles cases in which Early Termination should not be applied.

To the best of our knowledge, cgSpan is the first publicly available implementation for closed graphs mining. 
\end{abstract}


\section{Introduction}

The goal of Frequent Subgraph Mining (FSM) is to find subgraphs in a given labeled graphs set that occur more frequently than a given value. This value, known as support, is usually expressed as a percentage of the set size.  FSM algorithms can be designed to produce two types of output. The first type outputs all existing frequent subgraphs, while the second type only outputs closed frequent subgraphs.  A graph g is closed in a database if there exists no proper frequent supergraph of g that has equivalent occurrence with g. The first type of output may have two drawbacks. The first drawback is that the total number of frequent subgraphs discovered becomes very large. For example, if a Star frequent subgraph \cite{b1} with $k$ edges is discovered, all $2^k$ subgraphs of a Star graph have the same or a greater support and are therefore also discovered. The second drawback is that not closed frequent subgraphs can be of no interest to the task at hand. For example, frequent parts of the molecule are of no interest in mining chemical graphs set.

gSpan \cite{b2} is a popular FSM algorithm that discovers all frequent subgraphs. 
In this article, we introduce cgSpan, an efficient extension of gSpan that only detects closed frequent graphs.
cgSpan was developed to handle the practical use case of ETLs (Extract Transform Load) \cite{b3} refactoring. Each ETL can be modeled as a labeled graph. Many ETLs will share common subgraphs that implement the same logic, such as SSN (social security number) field detection and validation. cgSpan allows us to discover such repetitive logic and refactor it into a standalone ETL that is referenced by other ETLs. Such refactoring improves the maintenance, readability, and design of ETLs.

To date there are a number of gSpan implementations that use different programming languages \cite{b4} \cite{b5} \cite{b6}. Such implementations can of course be extended to cgSpan with relatively little programming effort. Our cgSpan implementation,\cite{b7}, extends the Python implementation \cite{b4}

CloseGraph \cite{ b8} was the first algorithm that was developed for frequent closed subgraphs extraction. cgSpan improves CloseGraph efficiency in two important ways:
\begin{enumerate}[label=(\roman*)]
\item cgSpan only examines extensions from the vertices on the right-most path to confirm that frequent subgraph is closed.
\\For the same purpose, ClosedGraph must examine extensions from all vertices.
\item cgSpan uses an efficient look-up table to check if early termination can be applied to the graph. Only a single lookup of the edge projections set of the last DFS code of the graph is required. After the lookup, the equivalent occurrence is checked only with a very limited number of closed graphs.
\\For the same purpose, CloseGraph must construct all possible extensions of the graph's parent, check every extension equivalent occurrence with a parent and compare each extension with the graph using the lexicographical order.
\end{enumerate}   
\IEEEpubidadjcol
Finally, we provide an efficient method to handle early termination failure.
We have discovered a number of different cases where applying early termination causes cgSpan to miss closed graphs. Such cases are detected and dealt with.

The rest of the paper is organized as follows. In Section \ref{section:PRELIMINARYCONCEPTS} we provide references to the definitions and notations used in the following sections. 

In Sections \ref{subsection:Orderofsupergraphdiscovery} and \ref{subsection:EarlyTerminationDetection} we establish the theoretical basis of the cgSpan algorithm.

Section \ref{subsection:EarlyTerminationImplementation} formulates the early termination algorithm of cgSpan.

Section \ref{subsection:EarlyTerminationFailure} formulates the method for the detection of early termination failure and the discovery of missing closed frequent subgraphs.

Finally, the cgSpan algorithm is provided in Section \ref{subsection:cgSpanImplementation}.

The results of the experiments are reported in Section \ref{section:EXPERIMENTSANDPERFORMANCESTUDY}. 

\section{Preliminary Concepts}\label{section:PRELIMINARYCONCEPTS}
The concepts used throughout this paper are listed below. Each concept is accompanied by references to the original definition in \cite{b2} and \cite{b8}.

\begin{definition}[{\Large labeled graph}] \cite[Definition.~1]{b2}, \cite[Section.~2]{b8}
A \emph{labeled graph} has labels associated with its edges
and vertices. We denote the vertex set of a graph \emph{g} by \emph{V(g)}, the edge set by \emph{E(g)}. A label function, \emph{l}, can map a vertex or an edge to a label.
\end{definition}
\begin{definition}[{\Large subgraph isomorphism}] \cite[Definition.~2]{ b2}, \cite[Definition.~1]{b8}
A subgraph isomorphism is an injective function $f\colon V(g) \to V(g^\prime)$, such that $(1)\forall u \in V(g), l(u) = l^\prime(f(u))$, and $(2)\forall (u,v) \in E(g), (f(u), f(v)) \in E(g^\prime)$  and $l(u, v) = l^\prime(f(u), f(v))$, where $l$ and $l^\prime$ are the label function of $g$ and $g^\prime$ respectively.
\end{definition}

\begin{definition}[{\Large occurrence}]. \cite[Definition.~5]{b8} Let $\varphi(g,g^\prime)$ represent the number of possible subgraph isomorphisms of $g$ in $g^\prime$. Given graph $g$ and graph dataset $D = \{G_1,G_2, ... ,G_n\}$, the occurrence of $g$ in $D$ is the sum of the number of subgraph isomorphisms of $g$ in every graph of $D$,
i.e. $\sum_{i=1}^{n} \varphi(g,G_i)$ denoted by $\mathcal{I}(g,D)$.
\end{definition}

\begin{definition}[{\Large graph extension}]. \cite[Section.~2]{b8} A graph $g$ can be extended by adding a new edge $e$. A new graph is denoted by $g \diamond_x e$.
\end{definition}

\begin{definition}[{\Large extendable subgraph isomorphism}]. \cite[Section.~4]{b8} Given a graph $g^\prime = g \diamond_x e$, $f$ a subgraph isomorphism of $g$ in $G$ and $f^\prime$ a subgraph isomorphism of $g^\prime$ in $G$. If $\exists \rho$, $\rho$ a subgraph isomorphism of $g$ in $g^\prime$, $\forall v f(v) = f^\prime(\rho(v))$, then we call $f$ extendable and $f^\prime$ an extended subgraph isomorphism from $f$.

We denote the number of such extendable $f$ by $\phi(g,g^\prime,G)$
\end{definition}

\begin{definition}[{\Large extended occurrence}]. \cite[Definition.~6]{ b8} Given graph $g^\prime =g \diamond_x e$ and graph dataset $D = \{G_1,G_2, ... ,G_n\}$, the extended occurrence of $g^\prime$ in $D$ w.r.t $g$ is the sum of the number of extendable
subgraph isomorphisms of $g$ (w.r.t $g^\prime$) in every graph among
$D$, i.e. $\sum_{i=1}^{n} \phi(g,g^\prime,G_i)$, denoted by $\mathcal{L}(g, g^\prime, D)$.
\end{definition}

\begin{definition}[{\Large equivalent occurrence}]. \cite[Section.~4]{b8} Given graph $g^\prime = g \diamond_x e$ and graph dataset $D$, if $\mathcal{I}(g,D) = \mathcal{L}(g, g^\prime, D)$, we say that $g$ and $g^\prime$ have the equivalent occurrence, which means wherever $g$ occurs in $D$, $g^\prime$ occurs.
\end{definition}

\begin{definition}[{\Large closed frequent subgraph mining}] \cite[Section.~2]{b8} 
If $g$ is a \emph{subgraph} of $g^\prime$, then $g^\prime$ is a \emph{supergraph} of $g$, denoted
by $g \subseteq g^\prime$ (\emph{proper supergraph}, if $g \subset g^\prime$). Given a labeled graph dataset, $D = \{G_1,G_2, ... ,G_n\}, support(g)$ (or $frequency(g)$) denotes the percentage (or number) of
graphs (in $D$) in which $g$ is a subgraph. The set of \textbf{frequent graph pattern}s, $FS$, includes all the graphs whose support is no less than a minimum support threshold, \emph{min\_sup}. The set of \textbf{closed frequent graph pattern}s, $CS$, is defined as follows:
\\ $CS=\{g|g \in FS\: and\: \nexists g^\prime \in FS\: such\: that\: g \subset g^\prime\: and\: g\: and\: g^\prime\: have\: equivalent\: occurrence\:\}$. 
\\ Since $CS$ includes no graph that has a proper supergraph with equivalent occurrence, we have $CS \subseteq FS$. The problem of Closed Frequent Subgraph Mining is to find the complete set of $CS$ in the graph dataset \emph{D} with a given \emph{min\_sup}.

Please note that definition of $CS$ in this article is different from $CS$ definition in \cite[Section.~2]{b8}. The definition in \cite[Section.~2]{ b8} is formulated as  $CS=\{g|g \in FS\: and\: \nexists g^\prime \in FS\: such\: that\: g \subset g^\prime\: and\: support(g) = support(g^\prime)\}$. The reason for this change in definition is that if $support(g) = support(g^\prime)$, but $\mathcal{I}(g,D) > \mathcal{L}(g, g^\prime, D)$, we consider $g$ to be a closed graph in $D$.

For example, in Figure~\ref{fig:figure-1}, $g^\prime_1$ is a supergraph of $g^\prime_2$ and both have $support=2$. However, $g^\prime_2$ occurs three times in $D$, while $g^\prime_1$ occurs only twice. Therefore, $g^\prime_2$ is considered to be a closed graph.  

In the rest of this paper we will simply refer to $CS$ as $S$.
\end{definition}

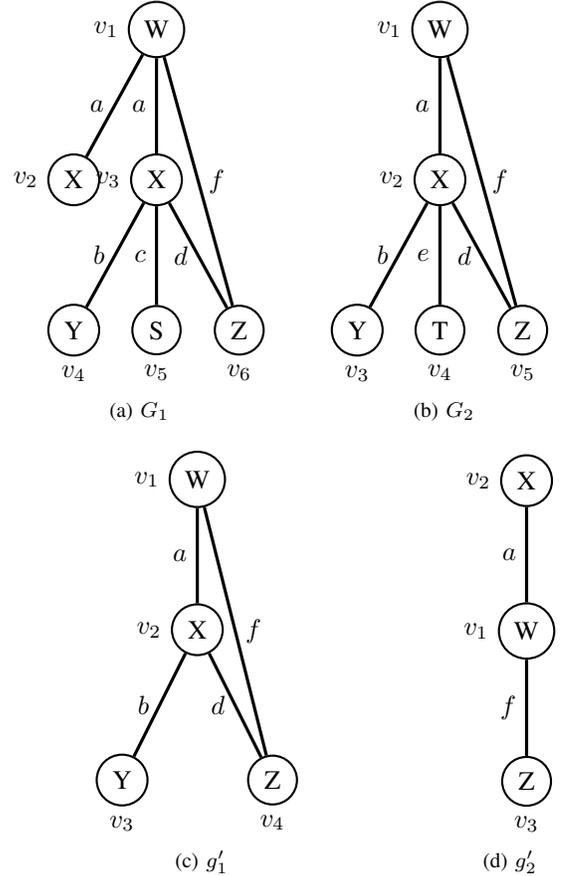
\begin{figure}[htb!]
\centering
%\begin{subfigure}[b]{0.2\textwidth}\centering
\subfloat[$G_1$]{%
\centering
\begin{tikzpicture}[
every edge/.style = {draw=black,very thick},
 vrtx/.style args = {#1/#2}{%
      circle, draw, thick, fill=white,
      minimum size=5mm, label=#1:#2}
                    ]
\node(A) [vrtx=left/$v_1$] at (0, 2) {W};
\node(B) [vrtx=left/$v_2$] at (-1.1, 0) {X};
\node(C) [vrtx=left/$v_3$] at (0, 0) {X};
\node(D) [vrtx=below/$v_4$] at (-1.1,-2) {Y};
\node(E) [vrtx=below/$v_5$] at (0,-2) {S};
\node(F) [vrtx=below/$v_6$] at (1.1,-2) {Z};
\path   (A) edge node[left]{$a$} (B)
        (A) edge node[left]{$a$} (C)
        (C) edge node[left]{$b$} (D)
 (C) edge node[left]{$c$} (E)
 (C) edge node[left]{$d$} (F)
 (A) edge node[right]{$f$} (F);
\end{tikzpicture}
%    \caption{$G_1$}
%    \end{subfigure}
    }
\hfil
%\begin{subfigure}[b]{0.2\textwidth}\centering
\subfloat[$G_2$]{%
\centering
\begin{tikzpicture}[
every edge/.style = {draw=black,very thick},
 vrtx/.style args = {#1/#2}{%
      circle, draw, thick, fill=white,
      minimum size=5mm, label=#1:#2}
                    ]
\node(A) [vrtx=left/$v_1$] at (1, 2) {W};
\node(B) [vrtx=left/$v_2$] at (1, 0) {X};
\node(C) [vrtx=below/$v_3$] at (-0.1,-2) {Y};
\node(D) [vrtx=below/$v_4$] at (1,-2) {T};
\node(E) [vrtx=below/$v_5$] at (2.1,-2) {Z};
\path   (A) edge node[left]{$a$} (B)
(B) edge node[left]{$b$} (C)
(B) edge node[left]{$e$} (D)
(B) edge node[left]{$d$} (E)
(A) edge node[right]{$f$} (E);
\end{tikzpicture}
%\caption{$G_2$}
%    \end{subfigure}
}
%\hfil 
%\medskip
\\
\hfil
\hfil
%\begin{subfigure}[b]{0.2\textwidth}\centering
\subfloat[$g_1^\prime$]{%
\centering
\begin{tikzpicture}[
		every edge/.style = {draw=black,very thick},
		vrtx/.style args = {#1/#2}{%
			circle, draw, thick, fill=white,
			minimum size=5mm, label=#1:#2}
		]
		\node(A) [vrtx=left/$v_1$] at (1, 2) {W};
		\node(B) [vrtx=left/$v_2$] at (1, 0) {X};
		\node(C) [vrtx=below/$v_3$] at (0,-2) {Y};
		\node(E) [vrtx=below/$v_4$] at (2,-2) {Z};
		\path   (A) edge node[left]{$a$} (B)
		(B) edge node[left]{$b$} (C)
		(B) edge node[left]{$d$} (E)
		(A) edge node[right]{$f$} (E);
	\end{tikzpicture}	
	%\caption{$g_1^\prime$}
%\end{subfigure}
}
\hfil
\hfil
\hfil
%\begin{subfigure}[b]{0.2\textwidth}\centering
\subfloat[$g_2^\prime$]{%
\centering
    \begin{tikzpicture}[
		every edge/.style = {draw=black,very thick},
		vrtx/.style args = {#1/#2}{%
			circle, draw, thick, fill=white,
			minimum size=5mm, label=#1:#2}
		]
		\node(A) [vrtx=left/$v_2$] at (1, 2) {X};
		\node(B) [vrtx=left/$v_1$] at (1, 0) {W};
		\node(C) [vrtx=below/$v_3$] at (1, -2) {Z};
		\path   (A) edge node[left]{$a$} (B);
		\path   (B) edge node[left]{$f$} (C);
	\end{tikzpicture}
	%\caption{$g_2^\prime$}
%\end{subfigure}
}
%\hfil

\caption{\textbf{Closed frequent graph pattern} $CS=\{g_1^\prime,g_2^\prime\}$ of $D = \{G_1,G_2\}$}
    \label{fig:figure-1}
\end{figure}

\begin{definition}[{\Large DFS Code}] \cite[Definition.~4]{b2}, \cite[Definition.~2]{b8} Given a DFS tree $T$ for a graph
$G$, an edge sequence $(e_i)$ can be constructed based on $\prec_{E,T}$,
such that $e_i \prec_{E,T} e_{i+1}$, where $i = 0 \ldots \lvert E\rvert - 1$. $(e_i)$ is called a DFS code, denoted as $code(G,T)$.
\end{definition}

\begin{definition}[{\Large DFS Lexicographic Order}] \cite[Definition.~5]{ b2}, \cite[Definition.~3]{b8} Suppose $Z = \{code(G,T)|T$ is a DFS tree of G$\}$, i.e., Z is a set containing all DFS codes for all the connected labeled graphs. Suppose there is a linear order $\prec_L$ in the label set $(L)$, then the lexicographic combination of $\prec_{E,T}$ and $\prec_L$ is a linear order $\prec_e$ on the set $E_T \times L \times L \times L$. \textbf{DFS Lexicographic Order} is a linear order defined as follows. If $\alpha = code(G_{\alpha},T_{\alpha}) = (a_0, a_1, \ldots , a_m)$ and $\beta = code(G_{\beta},T_{\beta}) = (b_0, b_1, \ldots , b_n),\alpha,\beta \in Z$, then $\alpha \leqslant \beta$ iff either of the following is true.
\begin{enumerate}[label=(\roman*)]
\item $\exists t, 0 \leqslant t \leqslant min(m, n), a_k = b_k for\: k < t,a_t \prec_e b_t$
\item $a_k = b_k for\: 0 \leqslant k \leqslant m, and\: n \geqslant m.$
\end{enumerate}
\end{definition}

\begin{definition}[{\Large Minimum DFS Code}]
\cite[Definition.~6]{b2}, \cite[Definition.~4]{b8}
Given a graph $G$, $Z(G) = \{code(G,T) | \forall T, \textrm{ T is a DFS tree of G} \}$, based on DFS lexicographic order, the minimum one, $min(Z(G))$,
is called \textbf{Minimum DFS Code} of $G$. It is also a canonical
label of $G$.
\end{definition}

\begin{definition}[{\Large DFS Code's Parent and Child}]\label{DFSCodeParent andChild} 
\cite[Definition.~7]{b2} 
Given a DFS code $\alpha = (a_0, a_1, \ldots , a_m)$, any valid DFS
code $\beta = (a_0, a_1 , \ldots , a_m, b)$, $\beta$ is called $\alpha$'s \textbf{child}, and $\alpha$ is called $\beta$'s \textbf{parent}.
\end{definition}

\begin{definition}[{\Large DFS Code Tree}]
\cite[Definition.~8]{b2}
In a DFS Code Tree, each node represents a DFS code, the relation between parent node and child node complies with the relation described in Definition \ref{DFSCodeParent andChild}. The relation between siblings is consistent with the DFS lexicographic order. That is, the pre-order	search of DFS Code Tree follows the DFS lexicographic order. The Tree is denoted as $\mathbb{T}$.
\end{definition}

\begin{definition}[{\Large DFS Code's Ancestors and Descendants}] 
\cite[Definition.~9]{b2}
Given two DFS codes, $\alpha$ and $\beta$, in $\mathbb{T}$, if there is a straight path from $\alpha$ to $\beta$, then $\alpha$ is called an ancestor of $\beta$, and $\beta$ is called a descendant of $\alpha$, denoted by $anc(\beta)$ = \{ all ancestors of $\beta$\}, and $des(\alpha)$ = \{ all descendants of $\alpha$\}.
\end{definition}

\begin{definition}[{\Large right-most extension}]
\cite[Section.~3.2]{b8}
Given a graph $g$ and a DFS tree $T$ in $g$, $e$ can be extended from the right-most vertex connecting to any other vertices on the right-most path (backward extension); or $e$ can be extended from vertices on the right-most path and introduce a new vertex (forward extension). We call these two kinds
of restricted extension as right-most extension: denoted by $g \diamond_r e$.
\end{definition}
\section{cgSpan Algorithm}\label{section:cgSpanALGORITHM}
\subsection{Order of Supergraph Discovery}\label{subsection:Orderofsupergraphdiscovery}
\begin{theorem}\label{DiscoveryOrder}
Given two graphs $G$ and $G^\prime$, ${G \subset G^\prime}$, ($G^\prime$ is a proper supergraph of $G$),  
$\alpha=(a_1,a_2, \ldots, a_n)$ and $\beta=(b_1,b_2,\ldots, b_m),  m > n$ be the DFS codes of $G$ and $G^\prime$ respectively when they are discovered for the first time in the DFS Code Tree,
then one of the following holds:
\begin{enumerate}[label=(\roman*)]
\item \label{i} $a_k = b_k for\: 0 \leqslant k \leqslant n$ i.e. $\alpha$ is ancestor of $\beta$ 
\item \label{ii} $G^\prime$ is discovered for the first time before $G$ is discovered for the first time.
\end{enumerate}
\end{theorem}
\begin{proof}
As stated in \cite{b2} "According to the definition of Minimum DFS code, the first occurrence of DFS code of a graph in $\mathbb{T}$ (pre-order) is its minimum DFS code." Therefore $\alpha = min(\alpha)$ and $\beta = min(\beta)$.
\\ The proof is by induction on n, the length of $\alpha$.
\\ \textbf{Base Case:} $n=1$:
\\ if $a_1=b_1$, then \ref{i} is satisfied.
\\ if $a_1 \neq b_1$, then $a_1 > b_1$. This holds because exists $b_j, j > 1$ such that $a_1 = b_j$. If $a_1 < b_1$, we could construct another DFS code of $G^\prime$  $\gamma = (b_j, b_1^\prime, \ldots, b_{m-1}^\prime)$.
$\gamma < \beta$, which contradicts $\beta$ being minimum DFS code.
\\ Since $a_1 > b_1$, $\beta$ is constructed before $\alpha$, $G^\prime$ is discovered before $G$ and \ref{ii} is satisfied.
\\ \textbf{Inductive hypothesis}: Suppose the theorem holds for all values of $n$ up to some $k$, $k \geq 1$.
\\ \textbf{Inductive step}: Let $n=k+1$.
 \ref{i} or \ref{ii} hold for $1  \leq n \leq k$ and we need to show that \ref{i} or \ref{ii} hold for n = k + 1.
\\ Let $\gamma=min(parent(\alpha))$
\\ If \ref{ii} is true for $\gamma$, then $G^\prime$ is discovered before $\gamma$ (inductive hypothesis), $\gamma$ is discovered before or at the same time as $parent(\alpha)$ ($\gamma \leq parent(\alpha))$ and $parent(\alpha)$ is discovered before $\alpha$. Therefore $G^\prime$ is discovered before $\alpha$ and \ref{ii} is true.
\\ If \ref{i} is true for $\gamma$, then $\gamma=(b_1,b_2,\ldots,b_k)$. By $\gamma$ definition $parent(\alpha) \geq \gamma$.
\\ If $parent(\alpha) > \gamma$ then $\alpha$ is discovered after $anc(\gamma)$ (all ancestors of $\gamma$). $\beta \in anc(\gamma)$ and therefore $G^\prime$ is discovered before $\alpha$
\\ If $parent(\alpha) = \gamma$ then $\alpha = (b_1,b_2,\ldots, b_k, b_j), j \geq k + 1$.
\\If $j=k+1$ then \ref{i} is true.
\\If $j>k+1$ then $b_j > b_{k+1}$ and \ref{ii} is true.
(If $b_j > b_{k+1}$ was not true, we could construct another DFS code of $G^\prime$  $\delta = (b_1,b_2,\ldots, b_k, b_j,b_{k+2}^\prime, \ldots, b_{m}^\prime)$.
$\delta < \beta$, which contradicts $\beta$ being minimum DFS code.)  
\end{proof}

\subsection{Early Termination Detection}\label{subsection:EarlyTerminationDetection}

Lemma \ref{EarlyTerminationClosedEquivalentGraph} and Lemma \ref{EarlyTerminationClosedGraphsDiscovery} provide a theoretical basis for cgSpan early termination detection.

\begin{lemma}\label{EarlyTerminationClosedEquivalentGraph}
For each frequent graph $g_0$ in $D$, exist $g_1,g_2,\ldots,g_n n \geq 0$, such that:
\begin{enumerate}[label=(\roman*)]
	\item $g_n$ is a closed graph in $D$
	\item $g_{i+1} = g_i \diamond_x e_i$ i.e. $g_{i+1}$ is extension of $g_i$
	\item $g_i$ and $g_{i+1}$ have equivalent occurrence.
\end{enumerate}
We say that $g_0$ and each of $g_i 1 \leq i \leq n$ have \textbf{transitive equivalence occurrence}. 

\end{lemma}
\begin{proof}
	$ $\newline
	If $g_0$ cannot be extended to a graph with equivalent occurrence, then $g_0$ is closed by definition and the conditions for $n=0$ are met.
	
	Otherwise $g_0$ can be extended to a graph $g_1$, $g_1= g_0 \diamond_x e_0$, so that $g_0$ and $g_1$ have equivalent occurrence.
	
	By induction, $g_i$ is either a closed graph or can be extended to a graph $g_{i+1}$, $g_{i+1}= g_i \diamond_x e_i$, so that $g_i$ and $g_{i+1}$ have equivalent occurrence.
	
	Since in each induction step $i$ the extended graph $g_i$ is one edge larger than in the previous step, the maximum number of steps $n$ will not exceed $\max_{G \in D} \lvert E(G) \rvert$. $g_n$  cannot be extended to a graph with equivalent occurrence and is therefore closed.   
\end{proof}

\begin{lemma}\label{EarlyTerminationClosedGraphsDiscovery}
After the DFS tree search of the graph $g$ with a DFS code $(a_1,a_2, \ldots, a_n)$ has been completed, i.e. all graphs whose minimum DFS code starts with $(a_1,a_2, \ldots, a_n)$ are discovered, all closed graphs that include $g$, $\{g^\prime | g \subseteq g^\prime, g^\prime\: is\: closed\: in\: D\}$, are also discovered.
\end{lemma}
\begin{proof}
	Since every closed graph that contains $g$ is also a supergraph of $g$, this lemma follows directly from Theorem~\ref{DiscoveryOrder}.
\end{proof}

When a DFS code $s$ is right-most extended with an edge $e$, we have to decide whether a further extensions of a new graph, $s \diamond_r e$, leads to a closed graph discovery. If this is not the case, the further DFS tree right-most extension of $s \diamond_r e$ should be terminated.

If $s \diamond_r e$ is itself a closed graph in $D$, we do not terminate its right-most extension.

If $s \diamond_r e$ is not a closed graph in $D$, cgSpan checks all closed graphs discovered up to this point.

If one of these closed graphs, $g$, and $s \diamond_r e$ have a \textbf{transitive equivalence occurrence}, there is no need to extend $s \diamond_r e$ any further, and only extensions of $g$ must be examined to find closed graphs. The DFS tree search of $g$ was already completed and by Lemma \ref{EarlyTerminationClosedGraphsDiscovery} all closed graphs that include $g$ have already been discovered. 
In such a case, cgSpan terminates $s \diamond_r e$ right-most extensions. 

If up to this point, no closed graph which has \textbf{transitive equivalence occurrence} with $s \diamond_r e$ has been discovered, from Lemma \ref{EarlyTerminationClosedEquivalentGraph} we know that such a closed graph exists and from Lemma \ref{EarlyTerminationClosedGraphsDiscovery} we know that such a graph will be discovered when $s \diamond_r e$ will be further right-most extended. Therefore, $s \diamond_r e$ right-most extension should not be terminated in such a case.  

We can conclude that cgSpan never miss an opportunity to terminate DFS tree extension wherever possible.

Let's see how cgSpan early termination works by examining steps in DFS lexicographical search of $D = \{G_1,G_2\}$ from Figure \ref{fig:figure-1} as shown in Figure~\ref{fig:figure-2}.
The lexicographical generation order of the generated patterns is: $g_1$, $g_2$, $g_3$, $g_4$ and $g_5$. When $g_5$ is discovered, cgSpan decides whether early termination should be applied to $g_5$. This decision is based solely on a fact if $g_5$ has a transitive equivalent occurrence with any closed graph discovered so far. The only closed graph discovered before $g_5$ was discovered is $g_4$. $\mathcal{I}(g_5,D) = \mathcal{L}(g_5,g_3, D) = 2$ and therefore $g_5$ and $g_3$ have equivalent occurrence. $\mathcal{I}(g_3,D) = \mathcal{L}(g_3,g_4, D) = 2$ and therefore $g_3$ and $g_4$ have equivalent occurrence. Therefore $g_5$ has extended equivalent occurrence with $g_4$.
Since $g_5$ has extended equivalent occurrence with a closed graph $g_4$, further right-most extension of $g_5$ is early terminated.

\begin{figure}[htb!]
	%\captionsetup[subfigure]{labelformat=empty}
	\centering
	%\begin{subfigure}[b]{0.2\textwidth}\centering
	\subfloat[$g_1$]{
	\centering
		\begin{tikzpicture}[
			every edge/.style = {draw=black,very thick},
		vrtx/.style args = {#1/#2}{%
			circle, draw, thick, fill=white,
			minimum size=3mm, label=#1:#2}
		]
		\node(A) [vrtx=left/$W$] at (0.5, 1) {};
		\node(B) [vrtx=left/$X$] at (0.5, 0) {};
		\path   (A) edge node[left]{$a$} (B);
		\end{tikzpicture}
		%\caption{$g_1$}
	%\end{subfigure}
	}
	\hfil    
	%\begin{subfigure}[b]{0.2\textwidth}\centering
	\subfloat[$g_2$]{
	\centering
		\begin{tikzpicture}[
		every edge/.style = {draw=black,very thick},
		vrtx/.style args = {#1/#2}{%
			circle, draw, thick, fill=white,
			minimum size=3mm, label=#1:#2}
		]
		\node(A) [vrtx=left/$W$] at (0.5, 1) {};
		\node(B) [vrtx=left/$X$] at (0.5, 0) {};
		\node(C) [vrtx=below/$Y$] at (0,-1) {};
		\path   (A) edge node[left]{$a$} (B)
		(B) edge node[left]{$b$} (C);
		\end{tikzpicture}
		%\caption{$g_2$}
	%\end{subfigure}
	}
	%\hfil 
	%\medskip
	\\
	\subfloat[$g_3$]{
	\centering
	%\begin{subfigure}[b]{0.2\textwidth}\centering
		\begin{tikzpicture}[
			every edge/.style = {draw=black,very thick},
		vrtx/.style args = {#1/#2}{%
			circle, draw, thick, fill=white,
			minimum size=3mm, label=#1:#2}
		]
		\node(A) [vrtx=left/$W$] at (0.5, 1) {};
		\node(B) [vrtx=left/$X$] at (0.5, 0) {};
		\node(C) [vrtx=below/$Y$] at (0,-1) {};
		\node(E) [vrtx=below/$Z$] at (1,-1) {};
		\path   (A) edge node[left]{$a$} (B)
		(B) edge node[left]{$b$} (C)
		(B) edge node[left]{$d$} (E);
		\end{tikzpicture}
		%\caption{$g_3$}
	%\end{subfigure}
	}
	\hfil
	\subfloat[$g_4$]{
	\centering
	%\begin{subfigure}[b]{0.2\textwidth}\centering
		\begin{tikzpicture}[
		every edge/.style = {draw=black,very thick},
		vrtx/.style args = {#1/#2}{%
			circle, draw, thick, fill=white,
			minimum size=3mm, label=#1:#2}
		]
		\node(A) [vrtx=left/$W$] at (0.5, 1) {};
		\node(B) [vrtx=left/$X$] at (0.5, 0) {};
		\node(C) [vrtx=below/$Y$] at (0,-1) {};
		\node(E) [vrtx=below/$Z$] at (1,-1) {};
		\path   (A) edge node[left]{$a$} (B)
		(B) edge node[left]{$b$} (C)
		(B) edge node[left]{$d$} (E)
		(A) edge node[right]{$f$} (E);
		\end{tikzpicture}
		%\caption{$g_4$}
	%\end{subfigure}
	}
	\hfil
	\subfloat[$g_5$]{%
\centering
	\begin{tikzpicture}[
		every edge/.style = {draw=black,very thick},
		vrtx/.style args = {#1/#2}{%
			circle, draw, thick, fill=white,
			minimum size=3mm, label=#1:#2}
		]
		\node(A) [vrtx=left/$W$] at (0.5, 1) {};
		\node(B) [vrtx=left/$X$] at (0.5, 0) {};
		\node(C) [vrtx=below/$Z$] at (1,-1) {};
		\path   (A) edge node[left]{$a$} (B)
		(B) edge node[left]{$d$} (C);
	\end{tikzpicture}
	%\caption{$g_2^\prime$}
%\end{subfigure}
}
	\caption{\textbf{Pattern Generation Order} of $D = \{G_1,G_2\}$ from Figure \ref{fig:figure-1}}
	\label{fig:figure-2}
\end{figure}
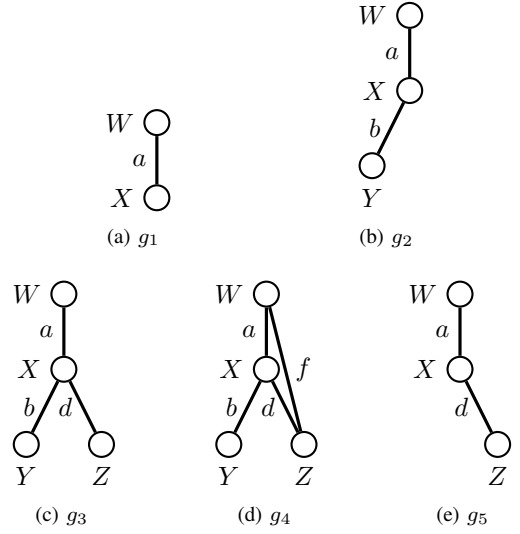

\subsection{Early Termination Implementation}\label{subsection:EarlyTerminationImplementation}

When a new graph $s \diamond_r e$ is discovered in the DFS search, cgSpan must check whether $s \diamond_r e$ has a transitive equivalent occurrence with any closed graph discovered so far and early terminate $s \diamond_r e$ DFS extension if such a closed graph exists.
In fact, we can limit this check to a small number of closed graphs discovered so far by maintaining a closed graphs hash table \cite{b9}.

Assume $s \diamond_r e$ has a transitive equivalent occurrence with a closed graph $g^\prime$. Let $\mathbb{F}=\{f\}$ and $\mathbb{F^\prime}=\{f^\prime\}$ be sets of isomorphisms of $s \diamond_r e$ and $g^\prime$ into $D=\{G_1,G_2, \ldots,G_n\}$ respectively. 
Then there exists an edge $e^\prime \in g^\prime$ such that $\{f^\prime(e^\prime), f^\prime \in \mathbb{F^\prime}\} = \{f(e), f \in \mathbb{F}\}$ i.e. $e^\prime$ and $e$ are injected into the same set of edges in $D$.

Therefore $s \diamond_r e$ has to be checked for having transitive equivalent occurrence only with closed graphs with such an edge $e^\prime$. 

Such sets of edges in $D$ are used as keys in the hash table of the closed graphs. As soon as the closed graph $g^\prime$ is discovered, we create a hash key $key_{e^\prime} = \{f^\prime(e^\prime), f^\prime  \in \mathbb{F^\prime}\}$ for each edge $e^\prime \in E(g^\prime)$ and add entries $(key_{e^\prime}, g^\prime)$ to the hash table of the closed graphs.

The hash table is denoted as $CGHT$ (Closed Graphs Hash Table).

To make the key hashable, we double index each edge in $D$ with $(i,j)$ where $i$ is an index of a graph $G_i, G_i \in D$ and $j$ is an edge index in $G_i$. The double index injective function $E(G), G \in D\rightarrowtail\mathbb{N}\times\mathbb{N}$ is denoted as $\mathbb{EE}$ (Edge Enumeration).

Discovered closed graphs are added to the closed graphs hash table using the \textsc{Add\_Closed\_Graph} procedure in Fig.~\ref{alg:hash_table}. The \textsc{Create\_Edge\_Hash\_Key} function in Fig.~\ref{alg:hash_table} is called to create a hash key.

Table~\ref{Tab:HashTable} shows the Edge Enumeration of $D = \{G_1,G_2\}$ in Figure~\ref{fig:figure-1} and the hash table of the closed graphs state after the closed graphs $g_1^\prime$ and $g_2^\prime$ of $D$ were discovered.

\begin{table}[!htb]
\captionsetup{justification=centering, labelsep=newline}
\caption{Closed Graphs Hash Table}
\label{Tab:HashTable}
\subfloat[\textbf{Edge Enumeration} of $D = \{G_1,G_2\}$ in Figure \ref{fig:figure-1}]{%
\label{subtab:p1}
%	\begin{subtable}{.415\linewidth}
		%\centering
		\begin{tabular}{ |c|c| } 
			\hline
			Edge & Enumeration \\ [0.5ex] 
			\hline
			$G_1(v_1, v_2)$ & $(1,1)$ \\
			\hline
			$G_1(v_1, v_3)$ & $(1,2)$ \\
			\hline
			$G_1(v_3, v_4)$ & $(1,3)$ \\
			\hline
			$G_1(v_3, v_5)$ & $(1,4)$ \\
			\hline
			$G_1(v_3, v_6)$ & $(1,5)$ \\
			\hline
			$G_1(v_1, v_6)$ & $(1,6)$ \\
			\hline
			$G_2(v_1, v_2)$ & $(2,1)$ \\
			\hline
			$G_2(v_2, v_3)$ & $(2,2)$ \\
			\hline
			$G_2(v_2, v_4)$ & $(2,3)$ \\
			\hline
			$G_2(v_2, v_5)$ & $(2,4)$ \\
			\hline
			$G_2(v_1, v_5)$ & $(2,5)$ \\
			\hline
		\end{tabular}
		%\caption{}
		%\label{subtab:p1}
%	\end{subtable}%
}
	%\begin{subtable}{.25\linewidth}
	%\begin{subtable}{.585\linewidth}
	\resizebox{0.25\textwidth}{!}{%
	\subfloat[\textbf{Closed Graphs Hash Table} state after the closed graphs $g_1^\prime$ and $g_2^\prime$ of $D$ were discovered]{%
	\label{subtab:HashTable}
		\centering
		\begin{tabular}{ |c|c| }
			\hline
			Key & Closed Graphs \\ [0.5ex] 
			\hline
			$\{(1,2),(2,1)\}$ & $g_1^\prime$ \\
			\hline
			$\{(1,3),(2,2)\}$ & $g_1^\prime$ \\
			\hline
			$\{(1,5),(2,4)\}$ & $g_1^\prime$ \\
			\hline
			$\{(1,6),(2,5)\}$ & $g_1^\prime, g_2^\prime$ \\ 
			\hline
			$\{(1,1),(1,2),(2,1)\}$ & $g_2^\prime$ \\
			\hline
		\end{tabular}
		}
		%\caption{}\label{subtab:p2}
%	\end{subtable} 
}
\end{table}

\begin{figure}[th!]
\begin{algorithmic}[1]
	\Function{Create\_Edge\_Hash\_Key}{$\mathbb{EE}, (v1,v2), \mathbb{F}$}
	\label{proc:proc_2}
	\Statex \hspace*{\algorithmicindent} \textbf{Input:} 
	\Statex \hspace*{\algorithmicindent} $\mathbb{EE}$ - edge enumeration of graphs dataset $D$
	\Statex \hspace*{\algorithmicindent} $(v1,v2)$ - edge
	\Statex \hspace*{\algorithmicindent} $\mathbb{F}$ - set of isomorphisms of $V, v_1,v_2 \in V$ into graphs dataset $D$
	\Statex \hspace*{\algorithmicindent} \textbf{Output:} 
	\Statex \hspace*{\algorithmicindent}$hash\_key$
	\State {$hash\_key \gets \{\}$}
	\ForAll{$f \in \mathbb{F}$}	
	\State {$hash\_key \gets hash\_key \cup EE((f(v_1), f(v_2)))$}
	\EndFor
	\State\Return $hash\_key$
	\EndFunction

	\Procedure{Add\_Closed\_Graph}{$\mathbb{CGHT}, \mathbb{EE}, g^\prime, \mathbb{F^\prime}$}
	\label{proc:proc_3}
	\Statex \hspace*{\algorithmicindent} \textbf{Input:} 
	\Statex \hspace*{\algorithmicindent} $\mathbb{CGHT}$ - closed graphs hash table
	\Statex \hspace*{\algorithmicindent} $\mathbb{EE}$ - edge enumeration of graphs dataset $D$
	\Statex \hspace*{\algorithmicindent} $g^\prime$ - closed graph
	\Statex \hspace*{\algorithmicindent} $\mathbb{F^\prime}$ - set of isomorphisms of $g^\prime$ into graphs dataset $D$
	\ForAll{$e^\prime \in E(g^\prime)$}	
	\State {$hash\_key \gets Create\_Edge\_Hash\_Key($}
	\State \hspace*{\algorithmicindent}{$\mathbb{EE}, e^\prime, \mathbb{F^\prime})$}
    \If{$CGHT[hash\_key] = \varnothing$}
    \State {$CGHT[hash_key] \gets \{g^\prime\}$}
    \Else
	\State {$CGHT[hash\_key] \gets CGHT[hash\_key] \cup g^\prime$}
    \EndIf
	\EndFor
	\State\Return
	\EndProcedure
	
\end{algorithmic}
\caption{Closed Graphs Hash Table} 
	\label{alg:hash_table}
\end{figure}

As soon as  $s \diamond_r e$ has to be checked for transitive equivalent occurrence with previously discovered closed graphs, we create a key $key_e = \{f(e), f \in \mathbb{F}\}$ and test transitive equivalent occurrences only with closed graphs, which are mapped by $key_e$ in the hash table of the closed graphs. This step is implemented by Line~\ref{proc:proc_4:line:hash_key_creation} and Line~\ref{proc:proc_4:line:closed_graphs_retrieval} of the \textsc{Early\_Termination}, Fig.~\ref{alg:early_termination}.\\
\\
To test whether $s \diamond_r e$ has transitive equivalent occurrence with a closed graph $g^\prime$, we must first find all possible isomorphisms $\mathbb{P} = \{\rho\}$ of $s \diamond_r e$ into $g^\prime$. To do this, we just have to choose an arbitrary isomorphism $f^\prime$ of $g^\prime$ into $G_i \in D$. Next we check all isomorphisms of $s \diamond_r e$ into $G_i \in D$. Every isomorphism $f$ of $s \diamond_r e$ into $G_i \in D$ that satisfies the condition $f(s \diamond_r e) \subset f^\prime(g^\prime)$ defines an isomorphism $\rho$ of $s \diamond_r e$ into $g^\prime$
$\rho(s \diamond_r e) = f^{\prime^{-1}}(f(s \diamond_r e))$. This step is implemented by lines ~\ref{proc:proc_4:line:find_isomorphisms_start} through ~\ref{proc:proc_4:line:find_isomorphisms_end} of the \textsc{Early\_Termination}, Fig.~\ref{alg:early_termination}.\\
\\
$s \diamond_r e$ and $g^\prime$ will have transitive equivalent occurrence if and only if one of the isomorphisms $\rho \in \mathbb{P}$ of $s \diamond_r e$ into $g^\prime$ satisfies the condition
$\forall f \in \mathbb{F}, \exists f^\prime \in \mathbb{F^\prime} f(s \diamond_r e) = f^\prime(\rho(g^\prime))$ i.e. wherever $s \diamond_r e$ occurs in $D$, $g^\prime$ must also occur exactly in the same place. If such an isomorphism $\rho$ is found for one of the closed graphs, an early termination should be applied to $s \diamond_r e$. 
This step is implemented by lines ~\ref{proc:proc_4:line:check_isomorphisms_start} through ~\ref{proc:proc_4:line:check_isomorphisms_end} of the \textsc{Early\_Termination}, Fig.~\ref{alg:early_termination}.\\
\\
For example, let's follow variable value assignments by \textsc{Early\_Termination} , Fig.~\ref{alg:early_termination}, in the processing of $D=\{G_1,G_2\}$ from Figure~\ref{fig:figure-1} when invoked with DFS code $\alpha=[(0,1,W,a,X), (1,2,X,d,Z)]$ and isomorphisms $f_1:V(\alpha) \to V(G_1), f_1(0)=v_1, f_1(1)=v_3, f_1(2)=v_6$  and $f_2:V(\alpha) \to V(G_2), f_2(0)=v_1, f_2(1)=v_2, f_2(2)=v_5$. The closed graphs hash table state in this invocation is shown in Table~\ref{Tab:HashTable} (b).\\
$hash\_key((1,2,X,d,Z)) \gets \mathbb{EE}((f_1(1), f_1(2))) \cup \mathbb{EE}((f_2(1), f_2(2))) = \mathbb{EE}(G_1(v_3, v_6)) \cup \mathbb{EE}(G_2(v_2, v_5)) = \{(1,5),(2,4)\}$\\
$G^\prime \gets \mathbb{CGHT}[\{(1,5),(2,4)\}] = \{g_1^\prime\}$\\
$\mathbb{F^\prime} \gets \{\\
f_1^\prime:V(g_1^\prime) \to V(G_1), f_1^\prime(v_1) = v_1, f_1^\prime(v_2) = v_3,  f_1^\prime(v_3) = v_4, f_1^\prime(v_4) = v_6\\
f_2^\prime:V(g_1^\prime) \to V(G_2), f_2^\prime(v_1) = v_1, f_2^\prime(v_2) = v_2,  f_2^\prime(v_3) = v_3, f_2^\prime(v_4) = v_5\}$\\
$\mathbb{P} \gets \{\rho:V(\alpha) \to V(g_1^\prime), \rho(0) = v_1, \rho(1) = v_2, \rho(2) = v_4\}$\\
$\forall v \in V(\alpha) f_1(v) = f_1^\prime(\rho(v))$ and $f_2(v) = f_2^\prime(\rho(v))$ therefore $true$ value is returned

\pagebreak

\begin{figure}[hbt!]
	\begin{algorithmic} [1]                   
		
		\Function{Early\_Termination}{$s \diamond_r e, \mathbb{F}, \mathbb{CGHT}, \mathbb{EE}$}
		\label{proc:proc_4}
		\Statex \hspace*{\algorithmicindent} \textbf{Input:}
		\Statex \hspace*{\algorithmicindent} $s \diamond_r e$ - graph checked for early termination
		\Statex \hspace*{\algorithmicindent} $\mathbb{F}$ - set of isomorphisms of $s \diamond_r e$ into graphs dataset $D$ 
		\Statex \hspace*{\algorithmicindent} $\mathbb{CGHT}$ - closed graphs hash table
		\Statex \hspace*{\algorithmicindent} $\mathbb{EE}$ - edge enumeration of graphs dataset $D$
		\Statex \hspace*{\algorithmicindent} \textbf{Output:} 
		\Statex \hspace*{\algorithmicindent} $true$ if early termination should be applied to $s \diamond_r e$ and $false$ otherwise. 
		\Statex \hspace*{\algorithmicindent} In case of $true$, also returns $g^\prime$ - the graph for which $s \diamond_r e$ has transitive equivalent occurrence and $\rho$ - the isomorphisms of $s \diamond_r e$ into $g^\prime$
		\State {$hash\_key \gets Create\_Edge\_Hash\_Key(\mathbb{EE}, e, \mathbb{F})$ \label{proc:proc_4:line:hash_key_creation}}
		\State {$G^\prime \gets \mathbb{CGHT}[hash\_key]$ \label{proc:proc_4:line:closed_graphs_retrieval}}
		\ForAll{$g^\prime \in G^\prime$}\label{closed_graphs}
		\State{$\mathbb{F^\prime} \gets$ isomorphisms of $g^\prime$ into $D$}
		\State {$\mathbb{P} \gets \varnothing$ \label{proc:proc_4:line:find_isomorphisms_start}}
		\State {select any $f^\prime$ from $\mathbb{F^\prime}$, $f^\prime:V(g^\prime) \to V(G_i), G_i \in D$}
		\ForAll{$f \in \mathbb{F}, f:V(s \diamond_r e) \to V(G_i)$}
		\If{$f(V(s \diamond_r e)) \subset f^\prime(V(g^\prime))$}
		\State {create $\rho:V(s \diamond_r e) \to V((g^\prime), \rho(v) = f^{\prime^{-1}}(f(v))$}
		\State {$\mathbb{P} \gets \mathbb{P} \cup \rho$ \label{proc:proc_4:line:find_isomorphisms_end}}
		\EndIf
		\EndFor
		\If{$\mathbb{P} = \varnothing$}
		\State \Goto{closed_graphs}	
		\EndIf	
		\ForAll{$\rho \in \mathbb{P}$ \label{proc:proc_4:line:check_isomorphisms_start}}
		\ForAll{$f \in \mathbb{F}$}
		\State {$ext\_subgraph\_isomorphism \gets false$}
		\ForAll{$f^\prime \in \mathbb{F^\prime}$}
		\If{$\forall v \in V(s \diamond_r e) f(v) = f^\prime(\rho(v))$}
		\State {$ext\_subgraph\_isomorphism \gets true$}
		\EndIf
		\EndFor
		\If{$ext\_subgraph\_isomorphism = false$}
		\State \Goto{proc:proc_4:line:check_isomorphisms_start}
		\EndIf
		\EndFor
		\State\Return $true, g^\prime, \rho$ \label{proc:proc_4:line:check_isomorphisms_end}
		\EndFor
		\EndFor
		\State\Return $false, \varnothing, \varnothing$
		\EndFunction
	\end{algorithmic}
	\caption{Early Termination} 
	\label{alg:early_termination}
\end{figure} 

\subsection{Handling Early Termination Failure}\label{subsection:EarlyTerminationFailure}
As stated in \cite{b8} there are special cases in which early termination cannot be applied.
One such example is provided in Figure~\ref{fig:figure-3}.

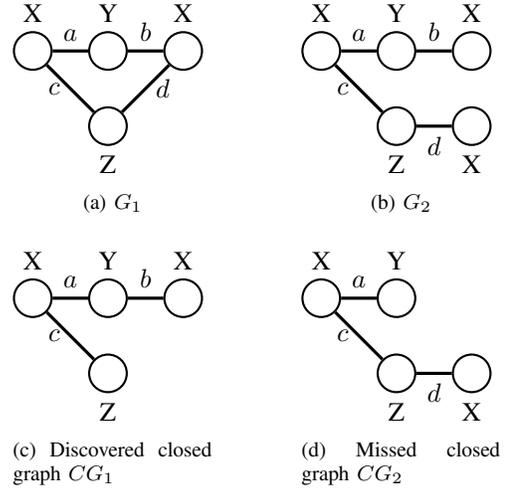
\begin{figure}[hbt!]
	\centering
	%\begin{subfigure}[b]{0.2\textwidth}
	\subfloat[$G_1$]{%
	%\centering
		\begin{tikzpicture}[
			every edge/.style = {draw=black,very thick},
			vrtx/.style args = {#1/#2}{%
				circle, draw, thick, fill=white,
				minimum size=5mm, label=#1:#2}
			]
			\node(A) [vrtx=above/X] at (-1, 1) {};
			\node(B) [vrtx=above/Y] at (0, 1) {};
			\node(C) [vrtx=above/X] at (1,1) {};
			\node(D) [vrtx=below/Z] at (0,0) {};
			\path   (A) edge node[above]{$a$} (B)
			(B) edge node[above]{$b$} (C)
			(C) edge node[right]{$d$} (D)
			(A) edge node[left]{$c$} (D);
		\end{tikzpicture}
		%\caption{$G_1$}
	%\end{subfigure}
	}
	\hfil    
	%\begin{subfigure}[b]{0.2\textwidth}
	\subfloat[$G_2$]{%
	%\centering
		\begin{tikzpicture}[
			every edge/.style = {draw=black,very thick},
			vrtx/.style args = {#1/#2}{%
				circle, draw, thick, fill=white,
				minimum size=5mm, label=#1:#2}
			]
			\node(A) [vrtx=above/X] at (-1, 1) {};
			\node(B) [vrtx=above/Y] at (0, 1) {};
			\node(C) [vrtx=above/X] at (1,1) {};
			\node(D) [vrtx=below/Z] at (0,0) {};
			\node(E) [vrtx=below/X] at (1,0) {};
			\path   (A) edge node[above]{$a$} (B)
			(B) edge node[above]{$b$} (C)
			(A) edge node[left]{$c$} (D)
			(D) edge node[below]{$d$} (E);
		\end{tikzpicture}
		%\caption{$G_2$}
	%\end{subfigure}
	}
	\hfil 
	\\
	%\begin{subfigure}[c]{0.2\textwidth}
	\subfloat[Discovered closed graph $CG_1$]{%
	%\centering
		\begin{tikzpicture}[
			every edge/.style = {draw=black,very thick},
			vrtx/.style args = {#1/#2}{%
				circle, draw, thick, fill=white,
				minimum size=5mm, label=#1:#2}
			]
			\node(A) [vrtx=above/X] at (-1, 1) {};
			\node(B) [vrtx=above/Y] at (0, 1) {};
			\node(C) [vrtx=above/X] at (1,1) {};
			\node(D) [vrtx=below/Z] at (0,0) {};
			\path   (A) edge node[above]{$a$} (B)
			(B) edge node[above]{$b$} (C)
			(A) edge node[left]{$c$} (D);
		\end{tikzpicture}
		%\caption{Discovered closed graph $CG_1$}
	%\end{subfigure}
	}
	\hfil 
	%\begin{subfigure}[c]{0.2\textwidth}
	\subfloat[Missed closed graph $CG_2$]{%
	%\centering
		\begin{tikzpicture}[
			every edge/.style = {draw=black,very thick},
			vrtx/.style args = {#1/#2}{%
				circle, draw, thick, fill=white,
				minimum size=5mm, label=#1:#2}
			]
			\node(A) [vrtx=above/X] at (-1, 1) {};
			\node(B) [vrtx=above/Y] at (0, 1) {};
			\node(D) [vrtx=below/Z] at (0,0) {};
			\node(E) [vrtx=below/X] at (1,0) {};
			\path   (A) edge node[above]{$a$} (B)
			(A) edge node[left]{$c$} (D)
			(D) edge node[below]{$d$} (E);
		\end{tikzpicture}
		%\caption{Missed closed graph $CG_2$}
	%\end{subfigure}
	}
	\caption{\textbf{Early Termination Failure, copied from \cite[Figure~5]{b8}}}
	\label{fig:figure-3}
\end{figure}

cgSpan can effectively handle early termination failure cases.
When a new minimum DFS code $\alpha = (a_0, a_1, \ldots , a_m)$ is constructed by a DFS search, cgSpan checks whether another DFS code $\beta$ exists  so that:
\begin{enumerate}[label=(\roman*)]
	\item $\beta$ should not be early terminated \label{item:etf_1}
	\item $G_\beta \subset G_\alpha$, $G_\beta$ and $G_\alpha$ are graphs subscripted by DFS codes $\beta$ and $\alpha$ respectively. \label{item:etf_2}
	\item $G_\beta \not\subset G_{parent(\alpha)}$, $G_\beta$ and $G_{parent(\alpha)}$ are graphs subscripted by DFS codes $\beta$ and $parent(\alpha)  = (a_0, a_1, \ldots , a_{m-1})$ respectively. i.e. $G_\beta$ includes the right-most vertex of $\alpha$ \label{item:etf_3}
	\item $\beta$ has not yet been discovered \label{item:etf_4}
\end{enumerate}
cgSpan does not construct $\beta$ explicitly, but rather verifies if such $\beta$ exists by examining each known early termination failure case conditions. 

For example, when the DFS code $\alpha = [(0,1,X,a,Y),(1,2,Y,b,X),(0,3,X,c,Z)]$ in DFS search of $D=\{G_1,G_2\}$ from Figure~\ref{fig:figure-3} is discovered, cgSpan detects that the edge $Y \overset{\text{b}} \longrightarrow X$ is breakable according to the definition in \cite{b8}. In this case $\beta$ is the DFS code of the graph created by removing vertex $2$ from $\alpha$.

Such DFS codes, $\alpha$, which should not be used to terminate other DFS codes are inserted into a separate database using the \textsc{Detect\_Early\_Termination\_Failure} procedure in Fig.~\ref{alg:early_termination_failure}.

The database of DFS codes can be efficiently implemented by a trie like data structure \cite{b10} to provide a quick search for the stored DFS codes.
\begin{enumerate}[label=(\roman*)]
	\item The root node of the trie always represents the null node.
	\item Each node (except the root) stores a DFS code 5-tuple.
	\item Child nodes are sorted in lexicographical order.
\end{enumerate}

After the Early Termination conditions in line \ref{proc:proc_4:line:check_isomorphisms_end} of \textsc{Early\_Termination} in Fig.~\ref{alg:early_termination} are met, cgSpan applies procedure \textsc{Reject\_Early\_Termination}  in Fig.~\ref{alg:early_termination_failure} to check whether an early termination should be rejected.

\textsc{Reject\_Early\_Termination} finds relevant prefixes of the DFS code of terminating closed graph $g^\prime$. If any of the prefixes exists in Early Termination Failure DFS codes trie storage, early termination is rejected.

For example, for the cgSpan execution on $D=\{G_1,G_2\}$ from Figure~\ref{fig:figure-3}, early termination conditions are met for 
$s=[(0,1,X,a,Y),(0,2,X,c,Z)]$, closed graph $CG_1$ with a DFS code $\alpha^\prime = [(0,1,X,a,Y),(1,2,Y,b,X),(0,3,X,c,Z)]$ and isomorphism $\rho= \{0 \rightarrow 0, 1 \rightarrow 1, 2 	\rightarrow 3\}$ from $s$ into $\alpha^\prime$.

As explained above, at this point $DFS\_Codes\_Trie$ already contains the DFS code $[(0,1,X,a,Y),(1,2,Y,b,X),(0,3,X,c,Z)]$.\\
In line \ref{proc:proc_5:line:projection} $s$ is projected into $\alpha^\prime$ using $\rho$. The result is a set of edges in $\alpha^\prime$ $\mathrm{A}=\{(0,1,X,a,Y), (0,3,X,c,Z)\}$.

Line \ref{proc:proc_5:line:max_index} computes the maximum index of edge in $\alpha^\prime$ that belongs to set $\mathrm{A}$. The edge $(0,3,X,c,Z)$ is such an edge and its index in $\alpha^\prime$ is $n=2$.

The next lines \ref{proc:proc_5:line:prefix_loop_start} through \ref{proc:proc_5:line:prefix_loop_end} check whether the $\alpha^\prime$ prefix $[(0,1,X,a,Y),(1,2,Y,b,X),(0,3,X,c,Z)]$ exists in $DFS\_Codes\_Trie$. Since $DFS\_Codes\_Trie$ contains the prefix $[(0,1,X,a,Y),(1,2,Y,b,X),(0,3,X,c,Z)]$, the early termination is rejected in line \ref{proc:proc_5:line:reject}.

\begin{figure}[hbt!]
\begin{algorithmic}[1]
	\Statex \Procedure{Detect\_Early\_Termination\_Failure}{
	$\alpha$, $DFS\_Codes\_Trie$}\label{proc:Detect_Early_Termination_Failure}
	\Statex \hspace*{\algorithmicindent} \textbf{Input:} 
	\Statex \hspace*{\algorithmicindent} $\alpha$ - DFS code
	\Statex \hspace*{\algorithmicindent} $DFS\_Codes\_Trie$ - Early Termination Failure DFS codes trie storage.
	\ForAll{known early termination failure case $s$}	
	\If{$\alpha$ is instance of $s$}
	\State add $\alpha$ to $DFS\_Codes\_Trie$
	\EndIf
	\EndFor
	\EndProcedure
	\Statex \Procedure{Reject\_Early\_Termination}
	{$s$, $g^\prime$, $\rho$, $DFS\_Codes\_Trie$}
	\label{proc:proc_5}
	\Statex \hspace*{\algorithmicindent} \textbf{Input:} 
	\Statex \hspace*{\algorithmicindent} $s$ - DFS code
	\Statex \hspace*{\algorithmicindent} $g^\prime$ - closed graph
	\Statex \hspace*{\algorithmicindent} $\rho$ - isomorphism of $s$ into $g^\prime$
	\Statex \hspace*{\algorithmicindent} $DFS\_Codes\_Trie$ - Early Termination Failure DFS codes trie storage.
	\State $\alpha^\prime = (a_0, a_1, \ldots , a_m) \gets$ DFS code of $g^\prime$
	\State  $\mathrm{A} \gets \{a_{i_k} | a_{i_k} \in \rho(s), 1 \leq k \leq | E(s)|\}$
	\label{proc:proc_5:line:projection}
	\State $n \gets \max\limits_{a_{i_k} \in \mathrm{A}}(i_k)$
	\label{proc:proc_5:line:max_index}
	\label{proc:proc_5:line:prefix_loop_start}
	\State $\alpha = (a_0, \ldots , a_n)$\Comment{$\alpha$ is prefix of $\alpha^\prime$ up to index $n$}
	\If{$\alpha \in DFS\_Codes\_Trie$}
	\State\Return{true}
	\label{proc:proc_5:line:reject}
	\EndIf
	\label{proc:proc_5:line:prefix_loop_end}
	\State\Return{false}
	\EndProcedure
\end{algorithmic}
\caption{Early Termination Failure} 
	\label{alg:early_termination_failure}   
\end{figure}

\subsection{cgSpan Implementation}\label{subsection:cgSpanImplementation}
cgSpan algorithm is provided in Fig.~\ref{alg:cgSpan}.
\begin{figure}[hbt!]
	\begin{algorithmic}[1]
	\Statex cgSpan($D, min\_sup, S$)
	\Statex \hspace*{\algorithmicindent} \textbf{Input:} graph dataset  $D$,  $min\_sup$.
	\Statex \hspace*{\algorithmicindent} \textbf{Output:} The closed frequent graph set $S$.	
		
		\State $S \gets \varnothing$\Comment{initialize closed frequent graph set}\label{line:initialization_start}
		\State $CGHT \gets \varnothing$\Comment{initialize closed graphs hash table}
		\State $DFS\_Codes\_Trie \gets \varnothing$\Comment{initialize early termination failure DFS codes trie}\label{line:initialization_end}
		\State create $EE$, the Edge Enumeration of $D$\label{line:edges_enumeration}
		\State $\mathbb{S}^1 \gets$ all frequent 1-edge graphs in $D$ together with isomorphisms $F_e$ of the graph into $D$ \label{line:S1_start}
		\State sort $\mathbb{S}^1$ in DFS lexicographic order\label{line:S1_end}
		\ForAll{edge $e \in \mathbb{S}^1$}	
		\State initialize $s$ with $e$
		\State \Call{Subgraph\_Mining}
		{$s$, $F_e$, $min\_sup$, $S$, $EE$, $CGHT$, $DFS\_Codes\_Trie$}
		\EndFor
	\Procedure{Subgraph\_Mining}
	{$s$, $F$, $min\_sup$, $S$, $EE$, $CGHT$, $DFS\_Codes\_Trie$}
		\If{$s \neq min(s)$}\label{line:MinimalDFSCodeStart}
		\State
		\Return
		\EndIf\label{line:MinimalDFSCodeEnd}
		\State $terminate\_early,g,\rho \gets$ \Call{Early\_Termination}{$s,F,CGHT,EE$}\label{line:EarlyTerminationStart}
		\If{$terminate\_early$}
		\If{$\neg$ \Call{Reject\_Early\_Termination}
		{$s$, $g$, $\rho$, $DFS\_Codes\_Trie$}}\label{line:EarlyTerminationFailure}
		\State
		\Return\label{line:EarlyTermination0}
		\EndIf
		\EndIf
		\State \Call{Detect\_Early\_Termination\_Failure}
		{$s$, $DFS\_Codes\_Trie$}\label{line:DetectEarlyTerminationFailure}
		\State $C \gets \varnothing$\label{line:recursion_start}
		\State scan $D$ once, find every edge $e$ such that
		$s$ can be \emph{right-most} extended to frequent $s \diamond_r e$
		\State $F_{s \diamond_r e} \gets $ isomorphisms of $s \diamond_r e$ into $D$
                \If{$support(s \diamond_r e) \geq min\_sup$}
		\State insert $s \diamond_r e$ and $F_{s \diamond_r e}$ into $C$;
                \EndIf
                \State sort $C$ in DFS lexicographic order
        \ForAll{$s \diamond_r e$ in $C$}
        \State \Call{Subgraph\_Mining}{$s \diamond_r e$, $F_{s \diamond_r e}$, $min\_sup$ , $S$ ,$EE$ , $CGHT$, $DFS\_Codes\_Trie$}
        \EndFor\label{line:recursion_end}        
		\If{$C = \varnothing$ or $\forall s \diamond_r e \in C, s$ does not have equivalent occurrence with$s \diamond_r e$} \label{line:EarlyTermination1}
		\State \Call{Add\_Closed\_Graph}{$CGHT,EE,s,F$}\label{AddHT}
		\State insert $s$ into $S$;
		\State\Return;
		\EndIf\label{line:add_closed_graph}
		\State\Return;
		\EndProcedure
	\end{algorithmic}
	\caption{cgSpan algorithm} 
        \label{alg:cgSpan}
\end{figure}

{\ \\}\textbf{Step 1 (line \ref{line:initialization_start}-\ref{line:initialization_end}):} Initializes data structures.
\\ \textbf{Step 2 (line \ref{line:edges_enumeration}):} Enumerates edges in $D$.
\\ Table \ref{Tab:HashTable}(\subref{subtab:p1}) shows an example of such enumeration
\\ \textbf{Step 3 (line \ref{line:S1_start}-\ref{line:S1_end}):} Adds all frequent 1-edge graphs in $D$ and their isomorphisms into $D$ to $\mathbb{S}^1$ and sorts them in DFS lexicographic order.
\\ After executing this step for $D=\{G_1,G_2\}$ from Figure~\ref{fig:figure-1}, $\mathbb{S}^1$ contains $[(0,1,W,a,X), (0,1,W,f,Z), (0,1,X,b,Y), (0,1,X,d,Z)]$ with their respective isomorphisms into $D$.
\\ \textbf{Step 4 (line \ref{line:MinimalDFSCodeStart}-\ref{line:MinimalDFSCodeEnd}):} As in gSpan, this step prunes non minimum DFS codes. 
\\ \textbf{Step 5 (line \ref{line:EarlyTerminationStart}-\ref{line:EarlyTermination0}):} This step first checks whether the conditions for early termination are satisfied. See subsection \ref{subsection:EarlyTerminationImplementation} for details. If the early termination conditions evaluate to true, checks whether early termination can be applied. See subsection \ref{subsection:EarlyTerminationFailure} for details.
If this is the case, the further extension of the DFS code $s$ is terminated.
\\ \textbf{Step 6 (line \ref{line:DetectEarlyTerminationFailure}):} Detects whether $s$ can cause an early termination failure. See subsection \ref{subsection:EarlyTerminationFailure} for details.
\\ \textbf{Step 7 (line \ref{line:recursion_start}-\ref{line:recursion_end}):} As in gSpan, finds all frequent right-most extension of $s$. Recursively calls \textsc{Subgraph\_Mining} for each right-most extension following extensions lexicographical order.
\\ \textbf{Step 8 (line \ref{line:EarlyTermination1}-\ref{line:add_closed_graph}):} If $s$ has no equivalence occurrence with any of it's right-most extensions $s \diamond_r e$, adds closed graph $s$ to the result set $S$ and to the closed graphs hash table $CGHT$.

\begin{theorem}
After executing cgSpan( $D, min\_sup, S$), graph $G \in S$ iff $G$ is a closed graph in $D$
\end{theorem}
\begin{proof}{\ \\}
Frequent subgraph $G$ with a minimum DFS code $s$ will not be added to $S$ only if line~\ref{line:EarlyTermination0} in algorithm~\ref{alg:cgSpan} is reached or step~\ref{line:EarlyTermination1} in algorithm~\ref{alg:cgSpan} evaluates to false for $s$.
\begin{itemize}
\item if
\\ Suppose that $G$ is a closed subgraph in $D$. It is enough to show that line~\ref{line:EarlyTermination0} is never reached by a prefix of $s$ and step~\ref{line:EarlyTermination1} evaluates to true for $s$.
\\ Line~\ref{line:EarlyTermination0} can only be reached if a prefix of $G$ DFS minimum code $s$ is early terminated by another closed graph. By definition this would be an early termination failure case. Line~\ref{line:EarlyTerminationFailure} guarantees that early termination failure cases do not reach line~\ref{line:EarlyTermination0}. 
\\ Since $G$ is a closed graph, its DFS Code $s$ has no right-most extensions with equivalent occurrence and step~\ref{line:EarlyTermination1} in algorithm~\ref{alg:cgSpan} evaluates to true.

\item only if
\\ Let $G$ be a frequent not closed graph in $D$. Since $G$ is not closed, it has an extended equivalent occurrence with a closed graph $G^\prime$. 
\\ According to Theorem~\ref{DiscoveryOrder}, either $G^\prime$  is ancestor of $G$ or $G^\prime$  is discovered before $G$.
\\ If $G^\prime$ is an ancestor of $G$, then step~\ref{line:EarlyTermination1} in Algorithm~\ref{alg:cgSpan} evaluates to false for $G$ ($G$ has right-most extension with equivalent occurrence) and $G$ is not added to $S$.
\\Let $F^\prime$ be the set of isomorphisms of $G^\prime$ into $D$ and $F$ be the set of isomorphisms of $G$ into $D$. If $G^\prime$  is discovered before $G$, line~\ref{AddHT} will add entries ($Create\_Edge\_Hash\_Key(EE,e,F^\prime)$, $G^\prime$) to $CGHT$ for every edge $e \in E(G^\prime)$ before $G$ is discovered. Let $e_s$ be the last edge in $s$. Let $e^\prime$ be $e_s$ matching edge in $G^\prime$. The entry ($Create\_Edge\_Hash\_Key(EE,e^\prime,F^\prime)$, $G^\prime$) $\in CGHT$ when $e_s$ is discovered for $G$. Since $G^\prime$ is an extended equivalent occurrence of $G$, the keys $Create\_Edge\_Hash\_Key(EE,e^\prime,F^\prime)$ and $Create\_Edge\_Hash\_Key(EE,e_s,F)$ are identical. Therefore, the call to \textsc{Subgrap\_Mining} with $s$ is guaranteed to reach line~\ref{line:EarlyTermination0} and $G$ is not added to $S$.
\end{itemize}
\end{proof}

\section{Experiments and Performance Study}\label{section:EXPERIMENTSANDPERFORMANCESTUDY}

In our experiments we use the two most famous datasets in subgraph mining, Chemical\_340 and Coumpounds\_422. Both datasets were obtained from the datasets database \cite{b11} of the SPMF open source library \cite{b12}.

The basic characteristics of the Chemical\_340 and Coumpounds\_422 datasets are summarized in Table~\ref{Tab:Datasets}.

\begin{table}[h!]
\captionsetup{justification=centering, labelsep=newline}
 \caption{Chemical\_340 and Coumpounds\_422 datasets \label{Tab:Datasets}}
%\begin{tabular}{|| m{0.11\textwidth}| m{0.035\textwidth} | m{0.051\textwidth}| m{0.051\textwidth} | m{0.041\textwidth}| m{0.041\textwidth} ||}
\resizebox{0.48\textwidth}{!} {%
\begin{tabular}{|| m{0.12\textwidth}| m{0.04\textwidth} | m{0.055\textwidth}| m{0.055\textwidth} | m{0.045\textwidth}| m{0.045\textwidth} ||}
 \hline
 Dataset Name&Graph count&Average node count per graph&Average edge count per graph&Vertex label count&Edge label count\\
 \hline\hline
Chemical\_340&340&27.02&27.40&66&4\\ 
 \hline
Coumpounds\_422&422&39.61&42.31&4&21\\
 \hline
 \end{tabular}%
}
 \end{table}
 
All experiments are done on a Intel(R) Core(TM) i7-7820HQ CPU @ 2.90GHz PC with 64.0 GB RAM, running 64-bit Windows 10.
 
cgSpan code is implemented with Python 3.6 and executed with PyPy 7.3.5.
    
cgSpan code is publicly available in \cite{b7}.

As part of our experiments, we have validated the completeness of early termination failure handling. The validation process was carried out by executing gSpan and filtering out all not closed frequent graphs from the gSpan output. The closed graphs set obtained by cgSpan execution was validated to be identical to the one obtained by gSpan execution and non closed graphs filtering.

Figure~\ref{Fig:MiningCompound} and Figure~\ref{Fig:MiningChemical} graphically depict the results of our tests of gSpan and cgSpan on Compounds\_422 and Chemical\_340. The data used to build Figure~\ref{Fig:MiningCompound} and Figure~\ref{Fig:MiningChemical} can be found in table~\ref{Tab:Compound Experiment Data} and table~\ref{Tab:Chemical Experiment Data} respectively.

In our experiments we found that the output of cgSpan can be roughly 10 percent the size of the output of gspan, and that the runtime of cgSpan can also be a fraction of the runtime of gSpan. The plots of cgSpan vs gSpan and frequent graphs vs closed frequent graphs in Figure~\ref{Fig:MiningCompound} should make our finding readily visible.

In other datasets cgSpan will continue to have a smaller output than gSpan; however, cgSpan may have a longer runtime than gSpan. The speed at which cgSpan completes compared to gSpan depends on the ratio of closed frequent graphs to frequent graphs in the provided dataset. For the Compounds\_422 dataset the ratio is low, as can be seen in Table~\ref{Tab:Compound Experiment Data} column 4, so cgSpan finishes much faster than gSpan. For the Chemical\_340 dataset the ratio is higher, see Table~\ref{Tab:Chemical Experiment Data} column 4, so cgSpan is slightly slower than gSpan. The aforementioned phenomena can be seen graphically in Figure~\ref{Fig:MiningChemical}.

Following our cgSpan vs gSpan testing we conducted further experiments on the value of early termination failure in the cgSpan algorithm. The results of experiments are in Table~\ref{Tab:Compound Early Termination Failure Data} and Table~\ref{Tab:Chemical Early Termination Failure Data}. We found that depending on the structure of the graphs in a given dataset early termination failure can be vitally important or inconsequential. For the Coumpounds\_422 dataset cgSpan with early termination failure handling can help detect almost 20 percent more graphs than cgSpan without early termination failure handling. For the Chemical\_340 dataset cgSpan with early termination failure handling found almost the exact same number of closed frequent graphs as did cgSpan without early termination failurehandling. Early termination failure handling is very valuable to the cgSpan algorithm as it helps guarantee the correctness of the algorithm.

cgSpan vs CloseGraph \cite{ b8} effectiveness can be concluded from the fact that cgSpan outperforms gSpan by a factor of 100 on Compounds\_422 dataset when $min\_sup$ is close to 5\%, while CloseGraph does the same only with a factor of 10. 
 
 \begin{figure}[hbt!]
	\centering
	%\begin{subfigure}{.24\textwidth}
	\subfloat[performance]{
	\centering
	\resizebox{0.24\textwidth}{!}{%
	\begin{tikzpicture}
    \begin{axis}[
        xlabel=minimum support,
        xtick={0.05, 0.06, 0.07, 0.08, 0.09, 0.1},
        ylabel=runtime(sec),
        ymode=log,
        xticklabel=\pgfmathparse{100*\tick}\pgfmathprintnumber{\pgfmathresult}\,\%,
        xticklabel style={/pgf/number format/.cd,fixed,precision=2}]
]
        \addplot[mark=x] coordinates {(0.05,101555.99) (0.06,2730.11) (0.07,1035.77) (0.08,306.62) (0.09,207.49) (0.1,187.44)};
        \addlegendentry{gSpan}
        \addplot[mark=o] coordinates {(0.05,1035.77) (0.06,134.43) (0.07,99.67) (0.08,80.22) (0.09,69.1) (0.1,58.54)};
        \addlegendentry{cgSpan}
    \end{axis}
    \end{tikzpicture}
    }%
    %\caption{performance}
    %\end{subfigure}%
    }
	%\begin{subfigure}{.24\textwidth}
	\subfloat[number of patterns]{
	\centering
	\resizebox{0.23\textwidth}{!}{%
		\begin{tikzpicture}
    \begin{axis}[
        xlabel=minimum support,
        xtick={0.05, 0.06, 0.07, 0.08, 0.09, 0.1},
        ylabel=Number of Patterns,
        ymode=log,
        xticklabel=\pgfmathparse{100*\tick}\pgfmathprintnumber{\pgfmathresult}\,\%,
        xticklabel style={/pgf/number format/.cd,fixed,precision=2}]
]
        \addplot[mark=x] coordinates {(0.05,885864) (0.06,111611) (0.07,62092) (0.08,24402) (0.09,17355) (0.1,15832)};
        \addlegendentry{frequent graphs}
        \addplot[mark=o] coordinates {(0.05,5489) (0.06,3148) (0.07,2374) (0.08,1856) (0.09,1479) (0.1,1246)};
        \addlegendentry{closed frequent graphs}
    \end{axis}
    \end{tikzpicture}
    }%
    %\caption{number of patterns}
	%\end{subfigure}
	}
	\caption{Mining Patterns in Coumpounds\_422} \label{Fig:MiningCompound}
\end{figure}
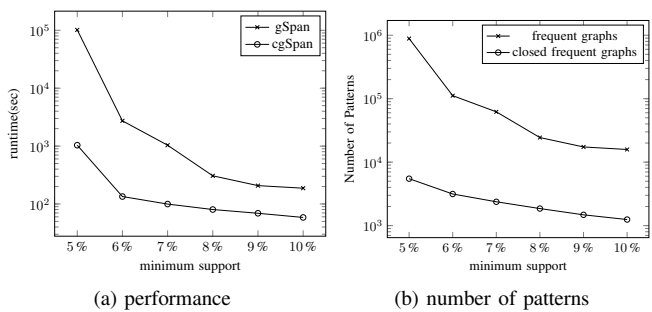

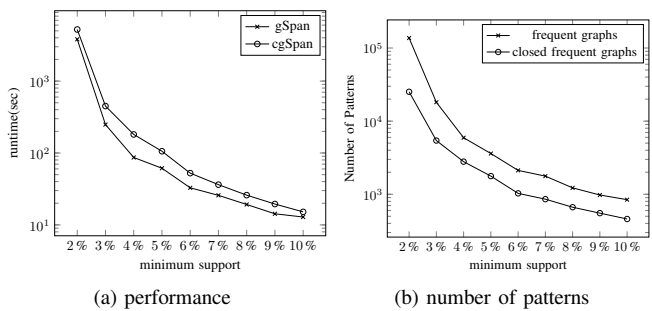
\begin{figure}[!hbt]
	\centering
	\subfloat[performance]{
	%\begin{subfigure}{.24\textwidth}
	\centering
	\resizebox{0.24\textwidth}{!}{%
	\begin{tikzpicture}
    \begin{axis}[
        xlabel=minimum support,
        xtick={0.02,0.03,0.04,0.05, 0.06, 0.07, 0.08, 0.09, 0.1},
        ylabel=runtime(sec),
        ymode=log,
        xticklabel=\pgfmathparse{100*\tick}\pgfmathprintnumber{\pgfmathresult}\,\%,
        xticklabel style={/pgf/number format/.cd,fixed,precision=2}]
]
        \addplot[mark=x] coordinates {(0.1,12.87)(0.09,14.27)(0.08,19.22)(0.07,25.77)(0.06,32.81)(0.05,61.36)(0.04,86.63)(0.03,248.3)(0.02,3808.79)};
        \addlegendentry{gSpan}
        \addplot[mark=o] coordinates {(0.1,15.17)(0.09,19.51)(0.08,25.93)(0.07,36.29)(0.06,52.46)(0.05,105.74)(0.04,181.17)(0.03,449.34)(0.02,5218.68)};
        \addlegendentry{cgSpan}
    \end{axis}
    \end{tikzpicture}
    }%
    %\caption{performance}
    %\end{subfigure}%
    }
	%\begin{subfigure}{.24\textwidth}
	\subfloat[number of patterns]{
	\centering
	\resizebox{0.23\textwidth}{!}{%
		\begin{tikzpicture}
    \begin{axis}[
        xlabel=minimum support,
        xtick={0.02,0.03,0.04,0.05, 0.06, 0.07, 0.08, 0.09, 0.1},
        ylabel=Number of Patterns,
        ymode=log,
        xticklabel=\pgfmathparse{100*\tick}\pgfmathprintnumber{\pgfmathresult}\,\%,
        xticklabel style={/pgf/number format/.cd,fixed,precision=2}]
]
        \addplot[mark=x] coordinates {(0.1,844)(0.09,977)(0.08,1224)(0.07,1770)(0.06,2121)(0.05,3608)(0.04,5935)(0.03,18121)(0.02,136949)};
        \addlegendentry{frequent graphs}
        \addplot[mark=o] coordinates {(0.1,459)(0.09,552)(0.08,665)(0.07,857)(0.06,1029)(0.05,1771)(0.04,2793)(0.03,5425)(0.02,25205)};
        \addlegendentry{closed frequent graphs}
    \end{axis}
    \end{tikzpicture}
    }%
    %\caption{number of patterns}
	%\end{subfigure}
	}
	\caption{Mining Patterns in Chemical\_340} \label{Fig:MiningChemical}
\end{figure}

\begin{table}[hbt!]
\captionsetup{justification=centering, labelsep=newline}
\caption{Coumpounds\_422 experiment data \label{Tab:Compound Experiment Data}}
\resizebox{0.48\textwidth}{!} {%
    \begin{tabular}{|| m{0.07\textwidth}| m{0.07\textwidth} | m{0.07\textwidth}| m{0.07\textwidth} | m{0.07\textwidth}| m{0.07\textwidth} | m{0.07\textwidth} | m{0.07\textwidth} ||} 
        \hline
        Percentage from Compound Dataset & Number of Frequent Graphs (gSpan) & Number of Closed Graphs (cgSpan) & Number of Closed Graphs / Number of Frequent Graphs & gSpan execution time & cgSpan execution time & cgSpan execution time / gSpan execution time \\
        \hline\hline
        10 & 15832 & 1246 & 0.0787 & 187.44 & 58.54 & 0.312 \\ 
        \hline
        9 & 17355 & 1479 & 0.0852 & 207.49 & 69.1 & 0.333 \\
        \hline
        8 & 24402 & 1856 & 0.0761 & 306.62 & 80.22 & 0.262 \\
        \hline
        7 & 62092 & 2374 & 0.0382 & 1035.77 & 99.67 & 0.096 \\
        \hline
        6 & 111611 & 3148 & 0.0282 & 2730.11 & 134.43 & 0.049 \\
        \hline
        5 & 885864 & 5489 & 0.0062 & 101555.99 & 1035.77 & 0.0102 \\
        \hline
    \end{tabular}%
    }
\end{table}

\begin{table}[hbt!]
\captionsetup{justification=centering, labelsep=newline}
\caption{Chemical\_340 experiment data \label{Tab:Chemical Experiment Data}}
    \resizebox{0.48\textwidth}{!} {%
    \begin{tabular}{|| m{0.07\textwidth}| m{0.07\textwidth} | m{0.07\textwidth}| m{0.07\textwidth} | m{0.07\textwidth}| m{0.07\textwidth} | m{0.07\textwidth} | m{0.07\textwidth} ||} 
        \hline
        Percentage from Chemical Dataset & Number of Frequent Graphs (gSpan) & Number of Closed Graphs (cgSpan) & Number of Closed Graphs / Number of Frequent Graphs & gSpan execution time & cgSpan execution time & cgSpan execution time / gSpan execution time \\
        \hline\hline
        10 & 844 & 459 & 0.5438 & 12.87 & 15.17 & 1.1787 \\ 
        \hline
        9 & 977 & 552 & 0.5650 & 14.27 & 19.51 & 1.3672 \\
        \hline
        8 & 1224 & 665 & 0.5433 & 19.22 & 25.93 & 1.3491 \\
        \hline
        7 & 1770 & 857 & 0.4842 & 25.77 & 36.29 & 1.4082 \\
        \hline
        6 & 2121 & 1029 & 0.4851 & 32.81 & 52.46 & 1.5989 \\
        \hline
        5 & 3608 & 1771 & 0.4909 & 61.36 & 105.74 & 1.7233 \\
        \hline
        4 & 5935 & 2793 & 0.4706 & 86.63 & 181.17 & 2.0913 \\
        \hline
        3 & 18121 & 5425 & 0.2994 & 248.3 & 449.34 & 1.8097 \\
        \hline
        2 & 136949 & 25205 & 0.1840 & 3808.79 & 5218.68 & 1.3701 \\
        \hline
    \end{tabular}*
    }
\end{table}

\begin{table}[hbt!]
\captionsetup{justification=centering, labelsep=newline}
\caption{Coumpounds\_422 early termination experiment data \label{Tab:Compound Early Termination Failure Data}}
\resizebox{0.48\textwidth}{!} {%
    \begin{tabular}{|| m{0.1\textwidth}| m{0.1\textwidth} | m{0.1\textwidth} | m{0.1\textwidth} ||} 
        \hline
        Percentage from Compound Dataset & Closed Graphs Found (cgSpan No Early Termination Failure) & Closed Graphs Found (cgSpan) & Closed Graphs (cgSpan) / Closed Graphs (No ETF) \\
        \hline\hline
        10 & 1092 & 1246 & 1.14  \\ 
        \hline
        9 & 1284 & 1479 & 1.15\\
        \hline
        8 & 1576 & 1856 & 1.18 \\
        \hline
        7 & 2008 & 2374 & 1.18 \\
        \hline
        6 & 2616 & 3148 & 1.20\\
        \hline
        5 & 4547 & 5489 & 1.21\\
        \hline
        4 & 13242 & 14698 & 1.11\\
        \hline
    \end{tabular}*
    }
\end{table}

\begin{table}[hbt!]
\captionsetup{justification=centering, labelsep=newline}
\caption{Chemical\_340 early termination experiment data} \label{Tab:Chemical Early Termination Failure Data}
    \resizebox{0.48\textwidth}{!} {%
    \begin{tabular}{|| m{0.1\textwidth}| m{0.1\textwidth} | m{0.1\textwidth} | m{0.1\textwidth} ||}
        \hline
        Percentage from Chemical Dataset & Closed Graphs Found (cgSpan No Early Termination Failure) & Closed Graphs Found (cgSpan) & Closed Graphs (cgSpan) / Closed Graphs (No ETF) \\
        \hline\hline
        10 & 459 & 459 & 1 \\ 
        \hline
        9 & 552 & 552 & 1\\
        \hline
        8 & 665 & 665 & 1 \\
        \hline
        7 & 857 & 857 & 1\\
        \hline
        6 & 1029 & 1029 & 1 \\
        \hline
        5 & 1765 & 1771 & 1.0034 \\
        \hline
        4 & 2764 & 2793 & 1.0105 \\
        \hline
        3 & 5363 & 5425 & 1.0116 \\
        \hline
    \end{tabular}*
    }
\end{table}

\section{Conclusions}\label{section:CONCLUSIONS}
We have shown that the gSpan algorithm can be efficiently extended to output only closed graphs.

For future work we consider the extension of cgSpan to handle directed graphs. In \cite{b13} the extension of gSpan to directed graphs is described. Since cgSpan is an extension of gSpan, the same approach can be used to extend cgSpan to directed graphs.

\section*{Acknowledgment}
We would like to express our gratitude to Prof. Doan AnHai for bringing us together to work on this project and providing support and guidance.


\begin{thebibliography}{00}

\bibitem{b1} Eric W. Weisstein.  Star graph. From MathWorld—A Wolfram Web Re-source.https://mathworld.wolfram.com/StarGraph.html.
\bibitem{b2} Xifeng Yan and Jiawei Han.  gspan:  Graph-based substructure patternmining.  In2002 IEEE International Conference on Data Mining, 2002.Proceedings., pages 721–724. IEEE, 2002.
\bibitem{b3} Informatica. What is extract transform load (etl)?\\
https://www.informatica.com/services-and-training/glossary-of-terms/extract-transform-load-definition.html.
\bibitem{b4} Chen Qingying and Karpov Nikolay. gspan.\\ https://github.com/betterenvi/gSpan, 2016.
\bibitem{b5} Nowozin Sebastian and Kudo Taku. gboost.\\ https://github.com/rkwitt/gboost, 2007.
\bibitem{b6} Tony Zhu. gspan.java. https://github.com/TonyZZX/gSpan.Java, 2018.
\bibitem{b7} Zevin Shaul and Sheikh Naaz.  cgspan.\\ https://github.com/NaazS03/cgspan, 2021.
\bibitem{b8} Xifeng Yan and Jiawei Han. Closegraph: Mining closed frequent graph pat-terns. InProceedings of the Ninth ACM SIGKDD International Conferenceon Knowledge Discovery and Data Mining, KDD ’03, page 286–295, NewYork, NY, USA, 2003. Association for Computing Machinery.
\bibitem{b9} Thomas H Cormen, Charles E Leiserson, Ronald L Rivest, and CliffordStein.Introduction to algorithms, page 253–280. Massachusetts Instituteof Technology., 3 edition, 2009.
\bibitem{b10} Edward Fredkin. Trie memory.Commun. ACM, 3(9):490–499, September1960.
\bibitem{b11} PhilippeFournier-Viger. Spmf Datasets. http://www.philippe-fournier-viger.com/spmf/index.php?link=datasets.php,2016.
\bibitem{b12} Philippe Fournier-Viger, Jerry Chun-Wei Lin, Antonio Gomariz, TedGueniche, Azadeh Soltani, Zhihong Deng, and Hoang Thanh Lam. Thespmf open-source data mining library version 2. InJoint European conference on machine learning and knowledge discovery in databases, pages36–40. Springer, 2016.
\bibitem{b13} Cane Wing-ki Leung. Technical notes on extending gspan to directedgraphs. Technical report, Technical Report, Management University, Sin-gapore, 2010.
\end{thebibliography}
\end{document}